%% file: main.tex
\documentclass[nohyperref]{article}
\usepackage{graphicx}
\input{macros}

\icmltitlerunning{The Hidden Power of Pure 16-bit Floating-Point Neural Networks}

\begin{document}

\twocolumn[

%\icmltitle{The Hitchhiker’s Guide to the 16-bit Neural Network: \\  Its power, pitfalls, and workaround}
\textcolor{red}{This paper has been updated and moved to \url{https://arxiv.org/abs/2305.10947}}

\icmltitle{The Hidden Power of Pure 16-bit Floating-Point Neural Networks} %\\  Its power, pitfalls, and workaround}

\begin{icmlauthorlist}

\icmlauthor{Juyoung Yun}{SBU}
\icmlauthor{Byungkon Kang}{SK}
\icmlauthor{Francois Rameau}{SK}
\icmlauthor{Zhoulai Fu}{SK} 

\end{icmlauthorlist}

\icmlaffiliation{SK}{Department of Computer Science, State University of New York Korea, Incheon, Republic of Korea}
\icmlaffiliation{SBU}{Department of Computer Science, Stony Brook University, NY, USA}

\icmlcorrespondingauthor{Byungkon Kang}{byungkon.kang@sunykorea.ac.kr}
\icmlcorrespondingauthor{Zhoulai Fu}{zhoulai.fu@sunykorea.ac.kr}

\vskip 0.3in
]

\printAffiliationsAndNotice{}

\begin{abstract}
%Lowering the precision of neural networks from the prevalent 32-bit precision
%has long been considered harmful to performance, despite the gain in space
%and  time. This paper investigates the unexpected performance gain of pure
%16-bit neural networks over the 32-bit networks in classification tasks. 
%We present extensive experimental results that favorably compare 
%various 16-bit neural networks' performance to those of the 32-bit models.
%In addition, a theoretical analysis on the efficiency of 16-bit models is 
%provided, which is coupled with empirical evidence to back it up.
%Finally, we discuss situations in which low-precision training is indeed
%detrimental.
%%presents an extensive experimental study of 16bit NN. We organize the work around three kinds of experiments (1) a situation where 16NN works exceptionally well, (2) a situation where 16NN is lacking (3) a workaround for lacking 16-bit NN. We believe this work will change the mindset of people about 16NN.  To ease users to reproduce our observations, we have supplied supplementary material on  our artifacts through the anonymous Gihub:
%%\begin{center}
%%	\textbf{\texttt{ www.github/artifact/}}
%%\end{center}

Lowering the precision of neural networks from the prevalent 32-bit precision has long been considered harmful to performance, despite the gain in space	and time. Many works propose various techniques to implement half-precision neural networks, but none study \emph{pure} 16-bit settings. This paper investigates the unexpected performance gain of pure 16-bit neural networks over the 32-bit networks in classification tasks. We present extensive experimental results that favorably compare various 16-bit neural networks’ performance to those of the 32-bit models. In addition, a theoretical analysis of the efficiency of 16-bit models is provided, which is coupled with empirical evidence to back it up. Finally, we discuss situations in which low-precision training is indeed detrimental.

\end{abstract}

\section{Introduction}

\input{intro}

\section{Related Work}
\label{relwork}
\input{relwork}

\section{Background}
\input{background}

\section{Theory}
\input{theory}

\section{Experiments}
\label{expr}

\input{method}

%\section{Discussion}
%\input{discussion}
% \section{Related Work}
% \label{relwork}
%\input{relwork}
\section{Conclusion}
\input{conclusion}
\bibliography{main}
\bibliographystyle{icml2023}

\newpage
\appendix
\onecolumn
\section{Extra experiments}
\subsection{Top-1 and time results}
\begin{figure*}[hbt!]
    \centering
    \caption{\small MNIST classification top-1 accuracy and computational time}
    \vskip 0.15in
    \includegraphics[width=\columnwidth]{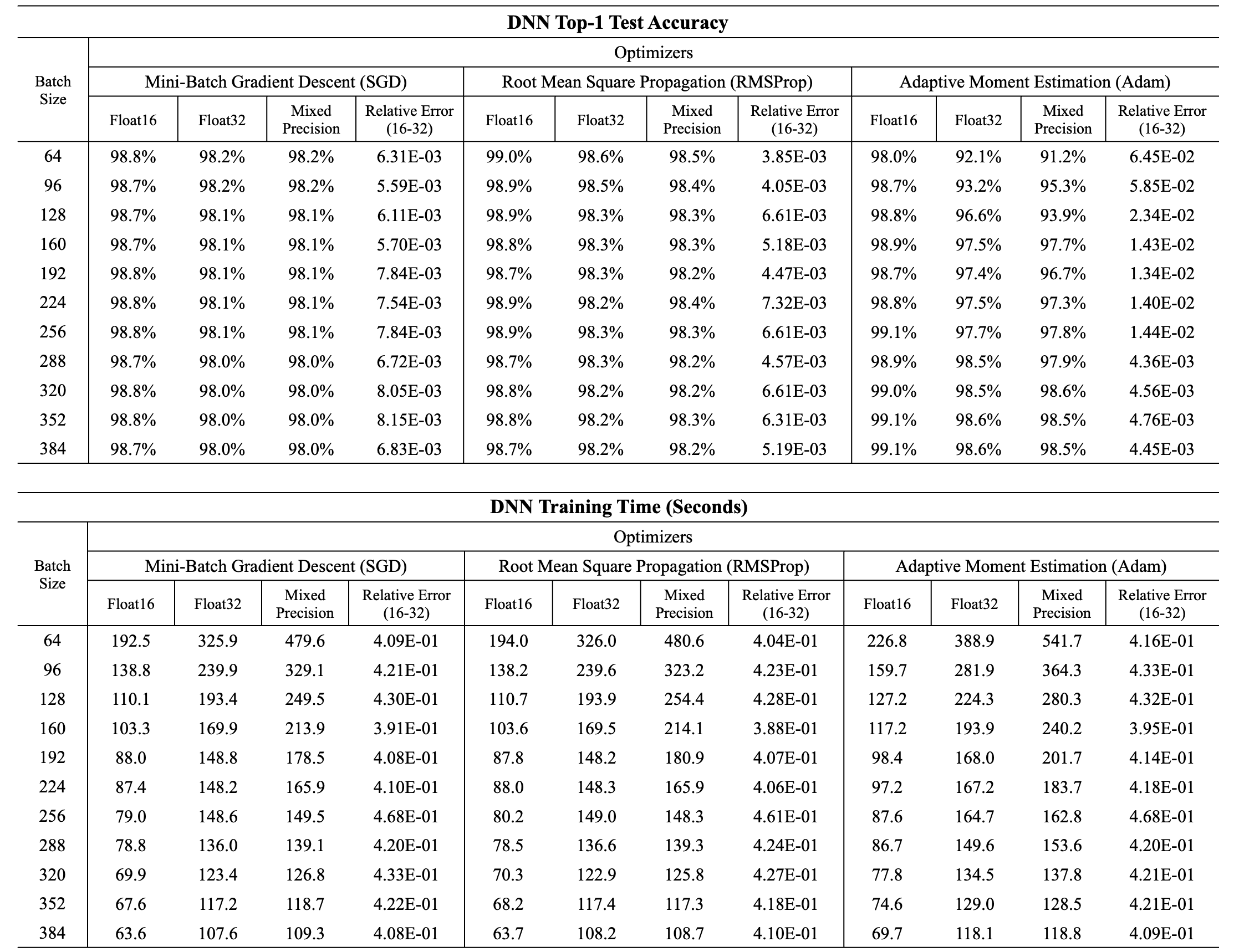}
    \label{fig:appendix1}
    \vskip -0.1in
\end{figure*} 

\begin{figure*}[hbt!]
    \centering
    \caption{\small CIFAR-10 classification top-1 and top-2 accuracy and computational time without BN Layers }
    \vskip 0.15in
    \includegraphics[width=\columnwidth]{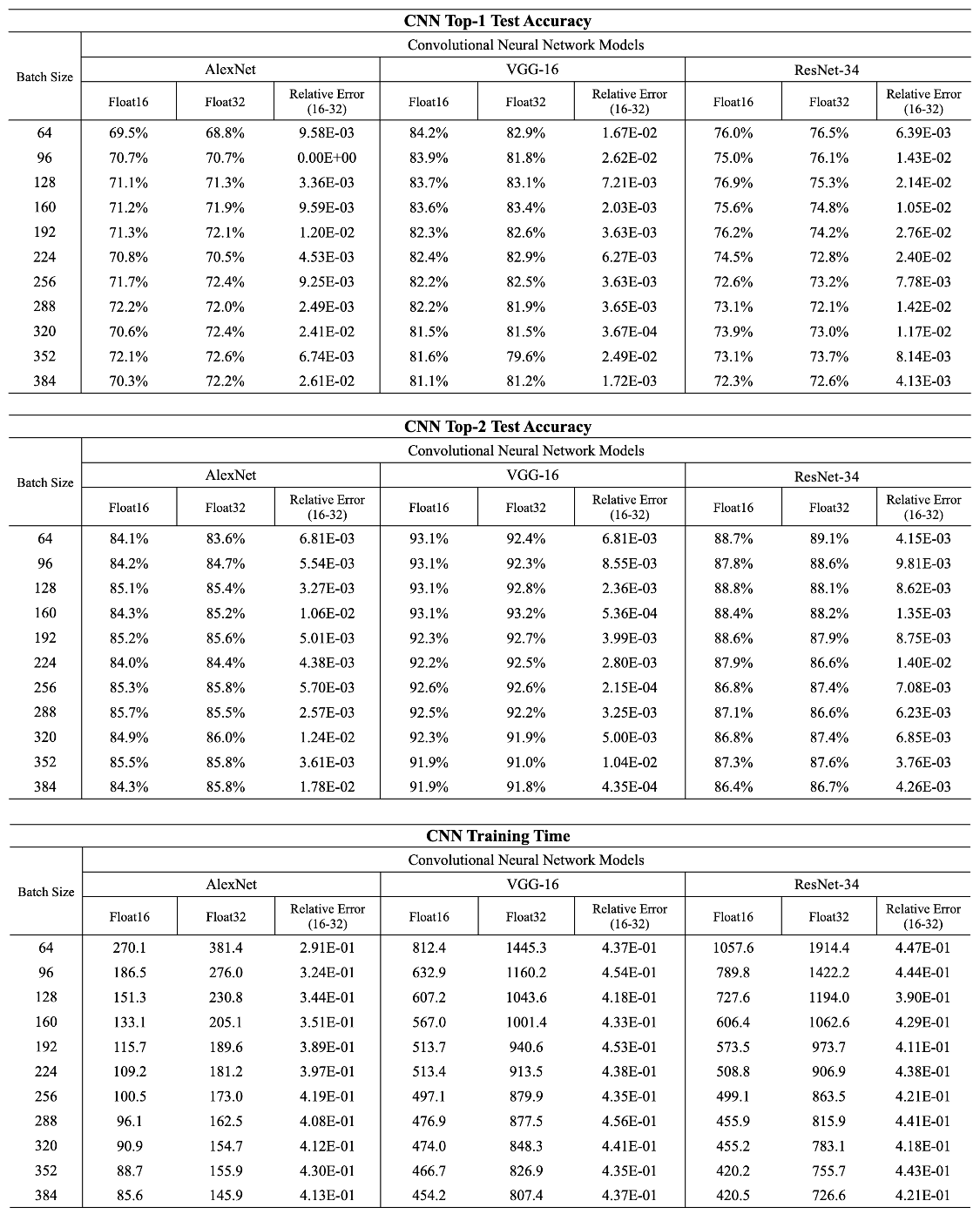}
    \label{fig:appendix2}
    \vskip -0.1in
\end{figure*} 
\subsection{Using 16-bit Batch Normalization}
As stated in the main text, the CNN results we provide in Section~\ref{expr}
are based on architectures without the batch normalization (BN) layers.
Figure~\ref{fig:appendix2} shows that 16-bit and 32-bit CNN models' performances without BN layers. The main reason for excluding those layers is that there are no off-the-shelf
16-bit implementations for BN layers. 

As a remedy, we implemented those manually and present the results.
We chose to move this result from the main text to here because we were not
able to fully verify that the implementation is truly a pure 16-bit one. 
That is, although we take every possible precaution to adhere to 16-bit
computation on the Python level, we do not know if any unprecedented
upcasting occurs in the library level (\textit{e.g.}, \texttt{libcuda} or
\texttt{cublas}). Hence, the results in Figure~\ref{fig:appendix3} are
gathered from 16-bit neural networks `to the best of our knowledge'.
\begin{figure*}[hbt!]
    \centering
    \vskip 0.15in
    \caption{\small Results from using 16-bit BN layers}
    \includegraphics[width=\columnwidth]{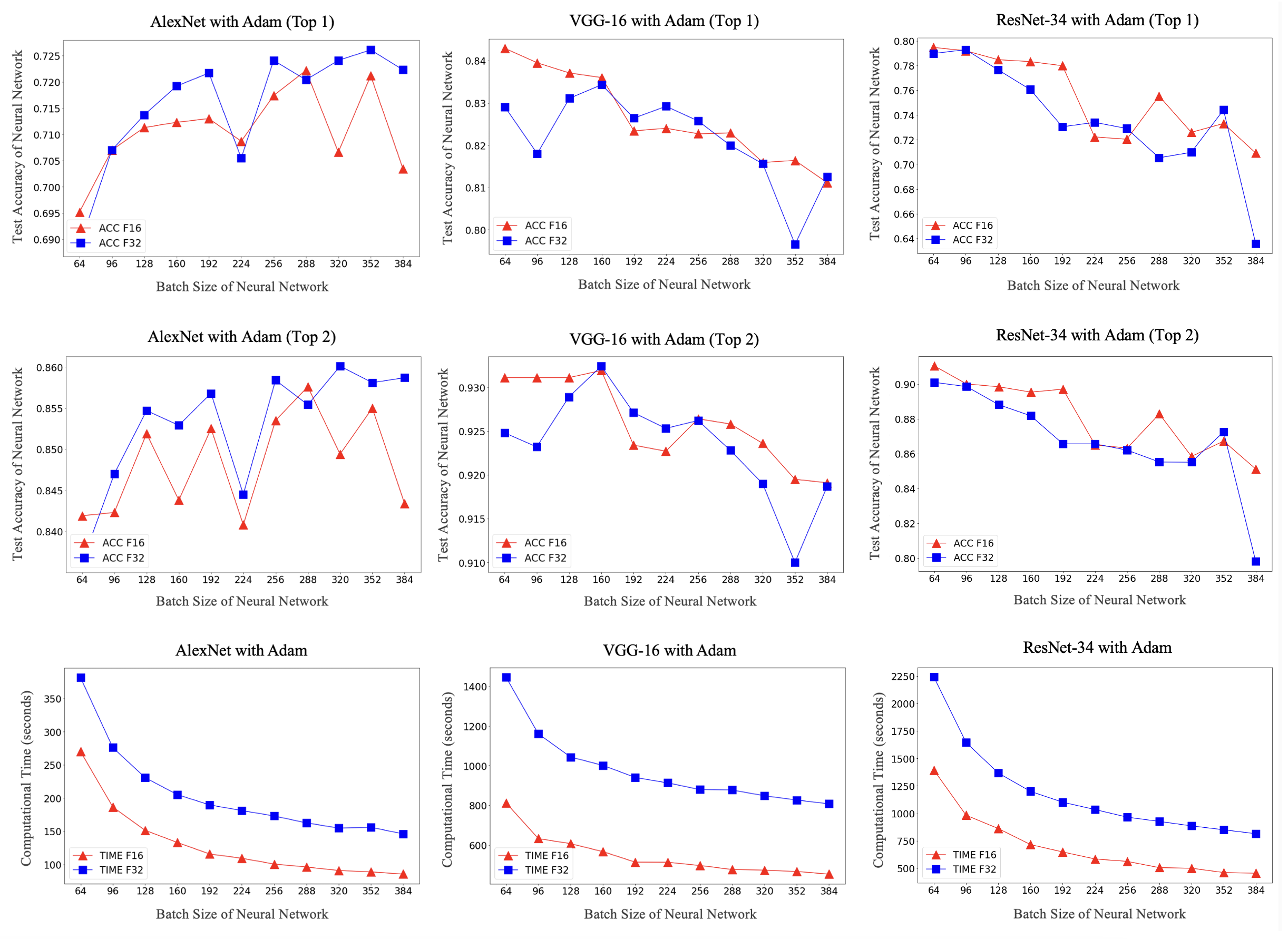}
    \label{fig:appendix3}
    \vskip -0.1in
\end{figure*}

\begin{figure*}[hbt!]
    \centering
    \caption{\small ResNet-34 with BN Layers CIFAR-10 classification top-1 and top-2 accuracy 
    and computational time}
    \vskip 0.15in
    \includegraphics[width=\columnwidth]{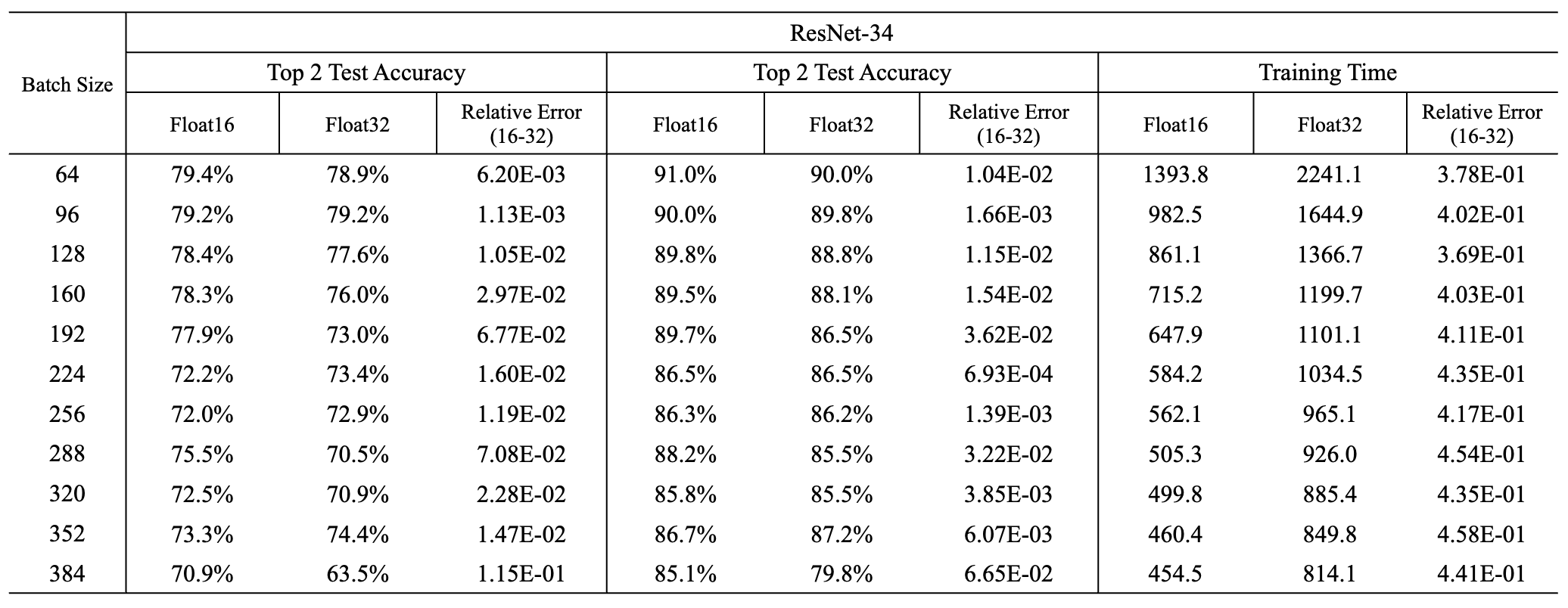}
    \label{fig:appendix4}
    \vskip -0.1in
\end{figure*} 

From the figures, we can see that incorporating 16-bit BN layers still 
results in a similar pattern as the one given in the main text. With the
exception of AlexNet, 16-bit networks give similar to superior performance
compared to the 32-bit models. 
\end{document}

%% file: macros.tex
\usepackage{microtype}
\usepackage{graphicx}
\usepackage{subfigure}
\usepackage{booktabs} % for professional tables
\usepackage{color,colortbl}
\usepackage[linewidth=1pt]{mdframed}

\usepackage{hyperref}

\usepackage[accepted]{icml2023}

\usepackage{amsmath}
\usepackage{amssymb}
\usepackage{mathtools}
\usepackage{amsthm}

\usepackage[capitalize,noabbrev]{cleveref}

\theoremstyle{plain}
\newtheorem{theorem}{Theorem}[section]

\theoremstyle{definition}
\newtheorem{lemma}[theorem]{Lemma}
\newtheorem{definition}[theorem]{Definition}

\theoremstyle{remark}

\usepackage[textsize=tiny]{todonotes}

\newcommand{\rounding}{\circ}

\newcommand{\real}{\mathbf{R}}
\newcommand{\fp}{\mathbf{F}}

\newcommand\mydef{\mathrel{\stackrel{\makebox[0pt]{\mbox{\normalfont\tiny def}}}{=}}}

%% file: intro.tex
Today's ubiquitous need for neural network techniques --- from autonomous vehicle driving, healthcare, and finance to general artificial intelligence and engineering --- has become a faith of fact for many. Significant computing power can be necessary for training neural networks on real-world data, which has stimulated the semiconductor industry to pursue cutting-edge chips and GPU solutions for reduced-precision floating-point arithmetic.

Reduced precision provides remarkable performance gain in speed, memory usage, and energy consumption over traditional CPU-based single and double-precision computing. Many GPUs are now powered with reduced-precision floating-point. In particular,  widely accessible NVIDIA GPUs support the IEEE-standardized half-precision, namely, the 16-bit floating-point format.

However, reduced precision alone is also known to cause accuracy loss. 
In the IEEE 16-bit format, positive numbers except subnormal ones lie between 5.96E-8 and 6.5E04, and thus a straightforward rounding or truncation of single or double precision data into half precision can cause overflow, underflow, or subnormal numbers, all of which affects numerical accuracy.

Consider the sigmoid function, $\mathit{sigmoid}(x) = 1/(1+e^{-x})$. It overflows if $x$ is small and underflows if $x$ is large, where both overflow and underflow occur in $e^{-x}$,  and the error will then be propagated to the function's output. As a quick experiment, one can set a vector $x$ to be all the 16-bit floating-point numbers (there are 63487 in total) excluding $\pm \infty$ and NaN, and $y$ be all the upcast 32bit float ({\tt y= np.float32(x)}), and element-wise compare  $\mathit{sigmoid} (x)$ and $\mathit{sigmoid} (y)$. They have a relative error of  4.85E-02 and an absolute error of  7.09E-05 on average.

Errors of such scale can be consequential for general scientific computing, e.g., causing the failure of the  Patriot missile system \cite{patriot}. The question  that would naturally arise for an ML practitioner would be:
Are floating-point errors of such scale too significant for a reduced precision model in machine learning to make the right prediction? 

In the machine learning community, many believe that "deep learning models ... are  "very tolerant of reduced-precision computations" \cite{9063049}.   ML researchers have actively investigated a wide range of techniques that  
lower the precision of  the floating point numbers used in neural networks but still maintain accuracy. For example, \citet{mixed-prec-training}  propose a mixed precision technique,  where weight, activation, and gradients are stored in 16 bits, but weight updates are carried out on 32 bits.   
\citet{https://doi.org/10.48550/arxiv.1905.12322} uses another mixed precision technique with 32-bit floating-point  and BFloat, Google Brain's half-precision  format, in which a tensor modification method is used to zero out the lower 16 bits of the 32-bit data flow.

All these algorithms use mixed precision, where two or more precision are chosen from a small number of available precisions,  typically half (16-bit)  and single (32-bit). To the best of our knowledge, the 
\emph{pure} 16-bit neural network is rarely used as a standalone solution  in the machine learning community.
This work targets the 16-bit pure neural network and studies whether we can use it out of the box, without parameter tuning or techniques commonly used in mixed precision algorithms like loss scaling. Our finding is positive. That is,  pure 16-bit neural networks, without any floating-point 32 components, despite being  imprecise by nature, can be precise enough to handle a major application of machine learning -- the classification problem.

Our work first formalizes the intuitive concept of error tolerance and proposes a lemma that  theoretically guarantees 16-bit "locally" achieves the same classification result as 32-bit under certain conditions. Combined with our preliminary observations, we conjecture that the 16-bit and 32-bit have close results in handling classification problems. We then validate our conjecture through extensive experiments on Deep neural network (DNN) and Constitutional Neural Network (CNN) problems. Our contributions follow:

\begin{itemize}

	\item  We aim to debunk the myth that plain 16-bit models do not work
	well. We demonstrate that training neural network models in 
	pure 16 bits with no additional measures to "compensate" results in 
	competitive, if not superior, accuracy. 
	
	\item  We  offer theoretical insights on \textit{why} half precision models work well, as well as empirical evidence that supports our analysis.
	\item We perform extensive experiments comparing the performance of various \textit{pure} 16-bit neural networks against that of 32-bit and mixed precision networks and find that the 16-bit neural networks can perform as well as their 32-bit counterparts. We also identify factors that could negatively influence the success of  16-bit models. %and demonstrate such cases through experiments.
\end{itemize}

%% file: relwork.tex
Several techniques have been proposed to reduce the precision of machine learning models while maintaining their accuracy to some degree. The approach that best aligns with our work is that of lowering the precision of the floating point numbers used in those models, but other approaches involve algorithm modification, fixed point formats, and hardware acceleration.\\
\textbf{Algorithm modification}. These approaches aim to modify certain components of the main algorithm to allow low-precision work. The work by~\cite{low-prec-sgd} offers a way to address the issues associated with low precision in gradient descent by assigning different precisions to the weights and the gradients. Such a scheme is assisted by special hardware implementation to speed up this mixed-precision process.
A similar approach is taken by~\cite{low-prec-rl} in reinforcement learning. This work proposes various mechanisms to perform reinforcement learning on low precision reliably. For example, adopting numerical methods to improve the Adam optimization algorithm so that underflow can be prevented. Many of such techniques proposed are somewhat ad-hoc to the target problem but might also be useful in a more general setting.\\ 
\textbf{Fixed point formats}. A \textit{fixed point format} is another number format often used to represent real numbers, although not standard. Several other works have adopted fixed point formats to allow low precision due to the intuitive representation and high extensibility to other configurations.
To verify the issues with low precision in deep learning, the authors of~\cite{dl-limprec} investigate the behavior of deep neural networks when the precision is lowered. In this work, the empirical results show that reducing the precision of the real numbers comprising the weights of the network shows a graceful degradation of accuracy.
\cite{fxpnet} performs neural network training in fixed point numbers by quantizing a floating point number into a low-precision fixed point number during training. The idea of quantizing floating point is also used in~\cite{fp-cnn}, where the authors approach the conversion as an optimization problem. The objective of the optimization is to reduce the network's size by adopting different bit-width for each layer. On a more system-based approach, \cite{shiftry} proposes compiler support for converting floating point numbers to low-precision fixed point numbers. \cite{kbsize} also takes compiler- and language-based support for achieving low-precision numbers.\\
Although not entirely the same as \textit{real} number fixed point, integer quantization can also be considered a special fixed point format and has become a viable choice in low-precision neural network design. In~\cite{int-dnn}, floating point weights are quantized to convert them into signed integer representations. The authors discover that adopting an integral quantization technique results in a regularization-like behavior, leading to increased accuracy. Similar to this approach, work by~\cite{mixed-prec-cnn} also proposes to use integer operations to achieve high accuracy. Such integer operations effectively convert the floating point numbers into what is known as dynamic fixed point (DFP) format. The work in \cite{quant-nn} also aims to quantize the weights to integers. However, the weights remain integers only during the forward pass and become floating point numbers for back-propagation to account for minute updates. Another work~\cite{bin-connect} takes an extreme quantization approach to binarize the weights to -1 and 1. The weights remain binary during the forward and backward pass but become floating points during the weight update phase. \cite{smoothquant} adopts a post-training quantization approach to reduce the memory footprint of large-scale language models.\\
\cite{flexpoint} proposes an adaptive numerical format that retains the advantages of floating and fixed point numbers. This is achieved by having a shared exponent that gets updated during training.\\
\textbf{Mixed precision}. On the other hand, \textit{mixed precision} approaches maintain standard floating point numbers in two different precision. The first practically successful reduced-precision floating point mechanism was proposed by~\cite{mixed-prec-training}. In this work, the authors devise a scheme to perform mixed precision training by maintaining a set of master full-precision floating point weights that serves as the `original copy' of the half-precision counterparts. While this technique does result in reduced running time, the mixed-precision nature of it limits the performance gain achieved.
\cite{ulp} propose a 4-bit floating point architecture mixed with a small amount of 8-bit precision. In addition to the format, the authors devise a two-phase rounding procedure to counter the low accuracy induced by the low-precision format.\\
\textbf{Hardware support}. Lastly, we would like to point out that many of these works either implicitly or explicitly require hardware support due to the individual floating/fixed point formats. Most approaches typically use FPGAs and FPUs to implement these formats, but other works such as \cite{bitfusion} propose novel architectures tailored to addressing the bit formats of the floating point numbers. Some  previously mentioned works, such as~\cite{ulp}, also hint at the possibility of leveraging hardware assistance.

Unlike these previous works, our work focuses explicitly on the IEEE floating point format, which is the de-facto standard for general computing machinery. More precisely, we investigate the pros and cons of using a pure 16-bit IEEE floating point format in training neural networks without  external support.

%% file: background.tex
\paragraph{General Notation}
The real numbers and integers are denoted by $\real$ and ${\bf Z}$, respectively. Given a vector $x=(x_0,..., x_{n-1})\in \real^n$, its  infinity norm or maximum norm:, denoted by $\vert x\vert_{\infty}$ , is the maximum element in the vector, namely $\vert x \vert_{\infty}\mydef\max_i|x_i|$.

\subsection{Floating-Point Representation}

We write $\fp_{16}$ to denote the set of 16-bit floating-point numbers excluding $\pm \infty$ and NaN (Not-a-Number). Following IEEE-754 standard \cite{4610935}, each  $x \in \fp_{16}$ can be written as  
\begin{align}\label{eq:16bit}
	x = (-1)^s \times g_0.g_1...g_{10} \mbox{ }_{(2)} \times 2^e
\end{align}	 
where $s\in \{0,1\}$, $g_i \in \{0,1\} $ ($0 \leq i \leq 10$), and $e\in {\bf Z}$. We call  $s$, $g_0.g_1....g_{10}$ and $e$ the sign, the significand, and the exponent, respectively. They satisfy  $g_0 \neq 0$ and $-14 \leq e \leq 15$. The case where $g_0 =0$ and $e=-14$ is called a subnormal number.\\   

We write $\fp_{32}$ for the set of 32-bit (single-precision)  floating-point numbers excluding  $\pm \infty$ and NaN.     The following property holds: 
\begin{align}
	\fp_{16} \subset \fp_{32} \subset \real
\end{align}
Namely, a 16-bit or 32-bit floating-point number is a real,  and a 16-bit can be exactly represented as a 32-bit (by padding with zeros). Tab. \ref{tab:fp} lists the range and the precision of 16-bit and 32-bit floating-point numbers.

\begin{table}[ht]\label{tab:fp}
	\centering
	\footnotesize
	\setlength{\tabcolsep}{4pt}
	\renewcommand{\arraystretch}{1.5}
	
	\caption{Some characteristics of the 16-bit  and 32-bit floating-point formats. }
	\vskip 0.15in
	\begin{tabular}{lllc}
		
		\toprule
		Type  & Size & Range & Machine-epsilon\\
		\midrule
		Half & 16 bits & 6.55E$\pm$4 & 4.88E-04 \\
		Single & 32 bits & 3.4E$\pm$38 & 5.96E-08  \\
		%double & 64 bits & 1.80E$\pm$308 & 1.11E-16 \\
		%quadruple & 128 bits & 1.19E$\pm$4932 & 9.63E-35 \\	
		\bottomrule
	\end{tabular}

\end{table}

\subsection{Floating-point Errors}
Rounding is necessary when representing real numbers that cannot be written as Eq. \ref{eq:16bit}. \emph{Rounding error} of a real $x$ at precision $p$ refers to $|x - \rounding_p(x)|$ where the rounding operation, $\rounding_p: \real \to\fp_p$, defines the nearest floating-point number of $x$. Namely, $\rounding_p(x) \mydef \mathrm{argmin}_{y \in \fp_p}  |x - y| $. \footnote{For simplicity, this definition of the rounding operation ignores the case where a tie needs to be broken.}
Rounding error is usually small, on the order of machine epsilon (Tab. \ref{tab:fp}), but it can be propagated and become more significant. For example, the floating-point code {\tt sin(0.1)} goes through three approximations. First, 0.1 is rounded to the floating-point $\rounding_p(0.1)$ for some precision $p$. Then, the rounding error is \emph{propagated} by  the floating-point code {\tt sin}. Lastly,  the  calculation output is rounded again if an exact representation is not possible.

Floating-point errors are usually measured in  \emph{absolute error} or \emph{relative error}. This paper focuses on classification problems where output numbers are probabilities between 0 and 1. Thus, we use the absolute error $|x - y|$  for quantifying the difference between two floating point numbers $x$ and $y$.

%% file: theory.tex
\begin{table*}[htp]\label{tab:mnist1}
	
	\footnotesize
	\setlength{\tabcolsep}{4pt}
	\renewcommand{\arraystretch}{1.5}
	
	\caption{The columns of "Floating-point errors" and "Error tolerance" refer to statistics of $\delta (M_{32}, M_{16}, x)$ and $\Gamma(M32, x)$ 	respectively, where $x$ ranges over the images of the MNIST dataset.} 
	\vskip 0.15in
	\begin{center}
		\begin{tabular}{llllllllll} 
			\toprule
			&  \multicolumn{4}{c}{Floating-point error} &  \phantom{abc}& \multicolumn{4}{c}{Error tolerance} \\
			\cmidrule{2-5} \cmidrule{7-10}
			Epochs  & Min & Max & Mean & Variance & & Min & Max & Mean & Variance \\
			\midrule		
			10 & 0.00E+00 & 1.68E-01 & 2.76E-03 & 5.62E-05 & &9.29E-05 & 1.00E+00 & 7.66E-01 & 7.58E-02 \\
			20 & 7.84E-15 & 2.07E-01 & 2.88E-03 & 8.53E-05 & &3.65E-05 & 1.00E+00 & 8.24E-01 & 6.31E-02 \\
			50 & 0.00E+00 & 3.82E-01 & 3.69E-03 & 1.97E-04 & &2.74E-06 & 1.00E+00 & 8.79E-01 & 4.72E-02 \\
			100 & 0.00E+00 & 5.64E-01 & 4.12E-03 & 3.27E-04 & &3.24E-04 & 1.00E+00 & 9.16E-01 & 3.42E-02 \\
			200 & 0.00E+00 & 6.75E-01 & 4.23E-03 & 4.76E-04 & &1.93E-04 & 1.00E+00 & 9.47E-01 & 2.19E-02 \\
			500 & 0.00E+00 & 9.35E-01 & 3.76E-03 & 6.18E-04 & &1.44E-03 & 1.00E+00 & 9.79E-01 & 7.72E-03 \\
			1000 & 0.00E+00 & 9.95E-01 & 3.14E-03 & 5.77E-04 & &1.91E-03 & 1.00E+00 & 9.91E-01 & 2.98E-03 \\
			
			\bottomrule
		\end{tabular}
		
	\end{center}
	
\end{table*}
Suppose $M_{16}$ and $M_{32}$ are 16-bit and 32-bit deep learning models trained by the same neural network architecture and hyper-parameters. By abuse of notation, we consider $M_{16}$ and $M_{32}$ as classifiers or functions that return the probability vector from the last layer  given an input $x$ (e.g., an image).

Let $x$ be an arbitrarily chosen input. Suppose $M_{32}(x)$ returns $(p_0,\cdots,p_{N-1})$, and $M_{16}(x)$ returns $(p'_0,\cdots,p'_{N-1})$. Clearly, the classification result of an input $x$ made by a classifier $M$ is given as 
\begin{align}
	\text{pred}(M, x)\mydef\mathrm{argmax}_{i}\{p_i\vert p_i\in M(x)\}.	
\end{align}
Due to the floating-point error, the classification results of $M_{32}$ and $M_{16}$ on $x$ can be different. To quantify this difference, we define the \textit{floating point error} as follows.
\begin{definition}
	Given the 16-bit classifer $M_{16}$, the 32-bit $M_{32}$, and an input $x$, the \emph{floating point error} between the classifiers is given as
	\begin{equation}
		\label{eq:delta}
		\delta (M_{32}, M_{16}, x)\mydef\vert M_{32}(x) - M_{16}(x)\vert_{\infty}
	\end{equation}
\end{definition}

The degree to which this difference affects the outcome is an important question we investigate in this work.
In fact, we can  show a sufficient condition (denoted by C) that guarantees the absence of difference between two classifiers.   
\begin{itemize}
	\item[{\bf (C.)}] If the difference between the largest of {$p_i$} and the second largest is greater than twice the floating point error, then the two classifiers  $M_{16}$ and $M_{32}$ have the same classification result on $x$. 
	
\end{itemize}
Illustration for condition (C): suppose $M_{32}(x)=(0.8, 0.1, 0.05, 0.05)$ for a classification problem of four labels.
Let the largest error between this probability vector and $M_{16}(x) $ be $\delta$. Then in the worst case, 0.8 can drop to $0.8 - \delta$ for the 16-bit, and the second largest probability becomes $0.1+\delta$. If $0.8-\delta > 0.1+\delta$, then $M_{16}$ and $M_{32}$ must have the same classification result on $x$, e.g., $M_{16}(x)=(0.7, 0.15, 0.1, 0.05)$. \\

Below, we formalize condition (C) following  introduction of the notion of  \textit{error tolerance}.
\begin{definition}
	The \emph{error tolerance} of a classifier $M$ with respect to an input $x$ is defined as  the gap between the largest probability and the second-largest one:
	\begin{equation}
		\Gamma(M,x)\mydef p_0 - p_1,
		\label{eq:gamma}
	\end{equation}
	where $p_0=\vert M(x))\vert_{\infty}$, and $p_1=\vert M(x)\backslash p_0\vert_{\infty}$.
\end{definition}
Here, $M_{32}(x)\backslash p_0$ refers to a vector of elements in $M_{32}(x)$ but with $p_0$ removed. The error tolerance can be thought of as quantifying the stability of the prediction. We have the following lemma corresponding to condition C.
\begin{lemma}\label{lem:twice}
	Consider a classification problem characterized by a pair $(X,Y)$ where $X$ is the space of input data, and $Y= \{0, .. N-1\}$  is the labels of classification. Suppose a learning algorithm trains a 32-bit model $M_{32}: X \to \fp_{32}^N $ and a	16-bit model $M_{16}: X \to \fp_{16}^N$ on a dataset $D\subseteq X \times Y$.
	
	We have: If  
	\begin{align} \label{eq:twice}
		\Gamma( M_{32},x) \geq 2 	\delta (M_{32}, M_{16}, x)
	\end{align}
	then $\mathrm{pred}(M_{32}, x)=\mathrm{pred}(M_{16}, x)$. %$\mathrm{classification} (M_{32},x)= \mathrm{classification} (M_{16}, x)$
\end{lemma}
\begin{proof}
	Let $M_{32}(x)$ be $(p_0, .... p_{N-1})$.  Without loss of generality we assume $p_0$ is the largest one in $\{p_i\}$ ($0\leq i \leq N-1$). 
	%Let  $||M_{32}(x) - M_{16}(x)||_{\infty}$ be denoted by $\delta$. 
	We denote $\delta(M_{32}, M_{16}, x)$ by $\delta$ hereafter.  Following Eq. \ref{eq:twice}, 	we have 
	\begin{align} \label{eq:gammabis}
		\forall i\in \{1,..., N-1\}, p_0  - p_i \geq  2  \delta.
	\end{align}
	Let $M_{16}(x)$ be $(p'_0, .... p'_{N-1})$. Then for each  $i\in \{1,..., N-1\}$, we have
	\begin{align*}
		p_0' &\geq p_0 - \delta  && \text{By Eq. \ref{eq:delta} and Def. of }  p_0, p_0'  \\ \nonumber
		&\geq p_i+  \delta && \text{By Eq. \ref{eq:gamma}  and  Eq. \ref{eq:gammabis}}\\ \nonumber
		&\geq p_i' &&\text{By Eq. \ref{eq:delta}}  \nonumber
	\end{align*}
	Thus $p_0'$ remains the largest in the elements of $M_{16}(x)$. 
\end{proof}         

\begin{table*}[htp]\label{tab:mnist2}
	\centering
	\footnotesize
	\setlength{\tabcolsep}{3pt}
	\renewcommand{\arraystretch}{1.5}
	
	\caption{Comparing accuracy and loss results between  32-bit and 16-bit neural networks on the MNIST dataset.}
	\vskip 0.15in
	\begin{center}
		
		\begin{tabular}{l  lllllllllll} 
			
			\toprule
			
			& \multicolumn{2}{c}{Train accuracy} &\phantom{abc} &  \multicolumn{2}{c}{Test accuracy}  &\phantom{abc} &  \multicolumn{2}{c}{Train loss}&\phantom{abc}  &   \multicolumn{2}{c}{Test loss} \\
			\cmidrule{2-3} \cmidrule{5-6}	 			\cmidrule{8-9} \cmidrule{11-12}
			Epochs & 32-bit & 16-bit && 32-bit & 16-bit && 32-bit & 16-bit && 32-bit & 16-bit \\

			\midrule	
			
			10 & 90.3\% & 90.0\% && 90.8\% & 91.0\% && 3.49E-01 & 3.69E-01 && 3.23E-01 & 3.39E-01 \\
			20 & 92.2\% & 92.3\% && 92.8\% & 92.8\% && 2.71E-01 & 2.94E-01 && 2.57E-01 & 2.74E-01 \\
			50 & 94.9\% & 95.3\% && 94.9\% & 94.7\% && 1.81E-01 & 2.10E-01 && 1.80E-01 & 2.04E-01 \\
			100 & 96.7\% & 96.4\% && 96.3\% & 95.8\% && 1.16E-01 & 1.49E-01 && 1.28E-01 & 1.52E-01 \\
			200 & 98.4\% & 97.3\% && 97.4\% & 97.3\% && 6.05E-02 & 9.32E-02 && 8.97E-02 & 1.10E-01 \\
			500 & 99.8\% & 99.1\% && 97.7\% & 98.1\% && 1.38E-02 & 4.00E-02 && 7.87E-02 & 8.54E-02 \\
			1000 & 100.0\% & 99.8\% && 97.8\% & 98.1\% && 2.97E-03 & 2.04E-02 && 9.02E-02 & 8.29E-02 \\
			
			\bottomrule
		\end{tabular}
		
	\end{center}

\end{table*}
Below we illustrate Lemma~\ref{lem:twice} through a simple neural network trained on MNIST. Our 32-bit implementation has three {\tt Dense} layers followed by a softmax layer at the end. Our  16-bit implementation uses the same architecture, except all floating-point operations are performed on 16-bit. \\
Table \ref{tab:mnist1} shows our results of error tolerance $\Gamma$ and floating-point error $\delta$. Observe that the mean floating-point error is of the magnitude of 1E-3 with a variance of 1E-5 or 1E-4; the error tolerance is 1E-1 with a variance of 1E-2. Thus, one can argue that Eq. \ref{eq:twice}, namely,  $\Gamma > 2\delta$,  holds for most data in MNIST. The table also shows that floating-point errors can be larger than the tolerance in some corner cases. Thus, we expect our 16-bit and 32-bit implementations to have close but different accuracy results. \\
Results at training and testing are presented. Table \ref{tab:mnist2} shows the accuracy and loss results in MNIST comparing 16-bit and 32-bit implementations. We can see consistently that the 16-bit results have accuracy close to those of the 32-bit models, sometimes even better. This result motivates us to study whether the 16-bit model is similar to the 32-bit model for more complex neural networks. 

In theory, if 80\% of data in a dataset satisfy Eq. \ref{eq:twice},  Lemma \ref{lem:twice} tells us that the 32-bit and 16-bit models will have at least 80\% of classification results being the same. The main challenge here is that we cannot determine if Eq. \ref{eq:twice} always holds or the percentage of data that satisfies it. We believe that for complex neural networks, most data meet Eq. \ref{eq:twice}. 
This is because the loss function for the classification problem is a cross-entropy in the form of $-\Sigma\log(p_i)$, which should guide $p_i$ toward $1$ during training and in turn, causes a large error tolerance $\Gamma$ compared to relatively small $\delta$. In fact,
Table~\ref{tab:mnist1} shows that the floating point errors ($\delta$'s) are nearly two orders of magnitude smaller than the $\Gamma$'s. Although the gap might close over the epochs, the difference remains sufficiently large to satisfy the condition of the lemma.\\  
With this theoretical development and  observations, we propose the following conjecture. 
\begin{mdframed}[leftmargin=10pt,rightmargin=10pt]
	%{\bf Conjecture:} 
	The accuracy  of a 16-bit neural network for classification problems, in the absence of significant errors involving floating-point overflow/underflow,  will be close to that of a 32-bit neural network.
\end{mdframed}
We anticipate a situation where floating-point errors can become significant due to overflow or underflow since 16-bit floating-point is known to have a smaller range (Table \ref{tab:fp}). The conjecture may be surprising, so we devote our next section to in-depth validation.

%% file: method.tex
We aim to compare the performance of 16-bit operations to 32-bit operations in deep neural network (DNN)\footnote{While DNNs subsume CNNs,  we use the term DNN to refer to fully-connected, non-convolutional neural networks.} and convolutional neural network (CNN) models.
we experiment over three CNN models, including AlexNet~\cite{alexnet}, VGG16~\cite{vgg}, and ResNet-34~\cite{resnet}. 
We will see how the different precision settings affect the computational time and accuracy of the models. Unlike case studies in the previous section, we use 100 epochs for all experiments and  gradually increase the batch size  from 64 to 384 to see how it affects 16-bit training.
All random seeds used in this study's experiments are fixed to facilitate the comparison. The experiments were conducted on NVIDIA’s RTX3080 Laptop GPU.

Table~\ref{tab:summary} gives the result most representative of our work.  A more detailed description and analysis of these results will follow in the subsequent subsections.

\begin{table}[hbt!]
	\centering
	\footnotesize

	\setlength{\tabcolsep}{3pt}
	\renewcommand{\arraystretch}{1.5}

	\caption{Summary of the time and accuracy performances of the three CNNs.}
	\vskip 0.15in
	\begin{tabular}{clrr}
		\toprule
		Model & Results & FP32 & FP16 \\% & Relative error \\
		\hline
		AlexNet & Time & 381s & 270s\\
		& Accuracy & 68.9\% & 69.5\% \\
		\hline
		VGG16   & Time & 1445s & 812s\\
		& Accuracy & 82.9\% & 84.3\% \\
		\hline
		ResNet-34 & Time & 1914s & 1058s\\
		& Accuracy & 76.6\% & 76.1\% \\
		\bottomrule
	\end{tabular}
	\label{tab:summary}
	\vskip -0.1in
\end{table}

\begin{figure*}[hbt!]
	\centering
	\caption{\small Top-1 accuracy (top row) and computational time (bottom row) on MNIST Classification using DNN}
	\vskip 0.15in
	\includegraphics[width=\linewidth]{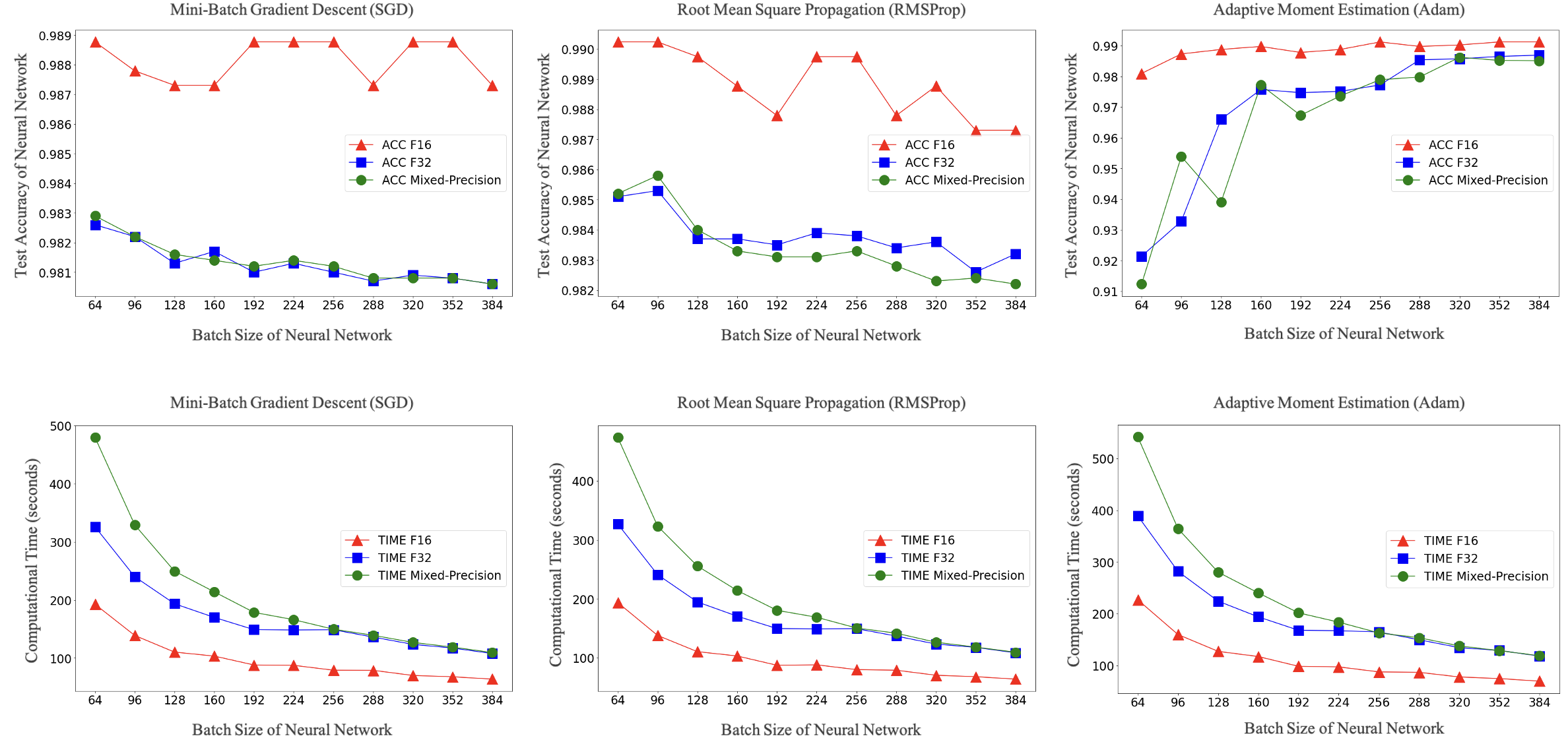}
	\label{fig:opt}
	\vskip -0.15in
\end{figure*} 

\begin{figure*}[hbt!]
	\centering
	\caption{\small Top-1 Accuracy, Top-2 Accuracy, and Computational Time for Cifar-10 Classification}
	\vskip 0.15in
	\includegraphics[width=2.0\columnwidth]{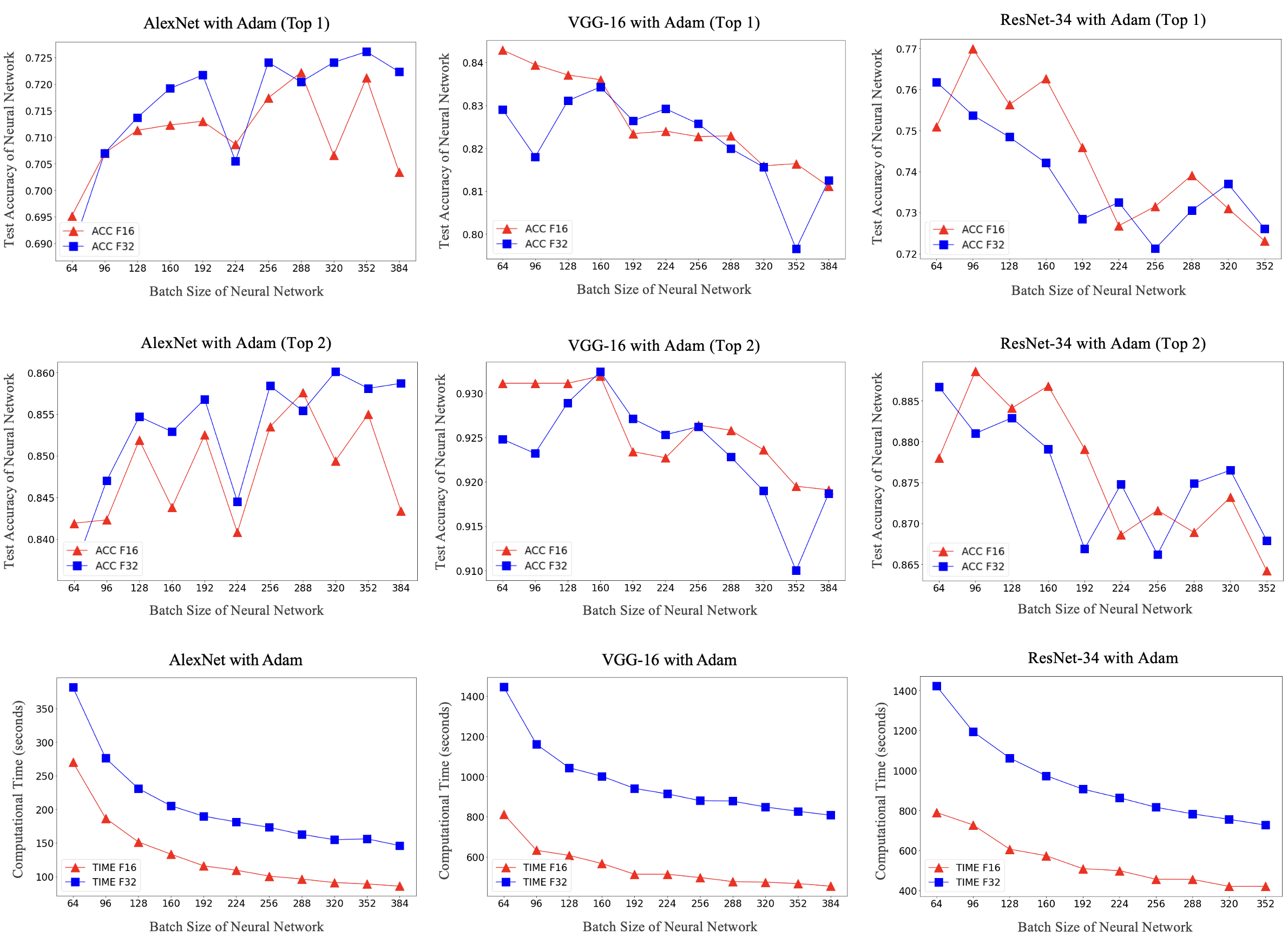}
	\label{fig:cnn}
	\vskip -0.15in
\end{figure*} 
\subsection{DNN Experiments}
In the DNN experiments, we train a DNN with three hidden layers
to perform the MNIST classification task and compare the performances between those of the 16-bit model and the 32-bit model. Each of the three layers in the DNN has 4096 neurons, whose outputs are fed to the next layer after passing through ReLU activation. Other works, such as~\cite{mixed-prec-training}, use a 32-bit softmax layer regardless of the overall precision settings to prevent potential numerical instability, but we leave stick with  a 16-bit softmax to see how a "pure" 16-bit model fares against 32-bit ones.
We also vary the types of optimizers among RMSProp, Adam, and SGD to examine their effects on performance. In addition, we confirm from this experiment that not only SGD but also other optimizers, such as RMSProp and Adam, can be used in 16-bit if $\epsilon$ is properly set.

In all experiments in this section, we use the learning rate of $10^{-3}$ and fix $\epsilon$ in RMSProp and Adam to $10^{-3}$ as well. We direct readers to the Appendix for experiments in other settings.
Finally, in addition to the 32-bit baseline, we also compare the mixed-precision training algorithm proposed by~\cite{mixed-prec-training} (as provided by TensorFlow).\\
Figure~\ref{fig:opt} shows that 16-bit deep neural networks are better than 32-bit and mixed precision in terms of computational time while maintaining similar test accuracy. In detail, (see Figure~\ref{fig:appendix1} in the Appendix) 16-bit SGD's computational time is decreased by 40.9\%, and the test accuracy was increased 0.6\%, respectively, compared to 32-bit while 16-bit computational time was decreased by 59.8\% and accuracy was increased 0.59\% compared to mixed precision when the batch size was the smallest at 64. Other batch sizes yielded consistent trends, albeit with smaller magnitude~\cite{batch}.

Optimization-wise, 16-bit RMSProp reduced the runtime by 40.4\% and 59.6\% compared to 32-bit and mixed, respectively, and accuracy is increased by 0.3\% and 0.4\%. 16-bit Adam improved the runtime by 41.6\% and 58.1\%, and 6.4\% and 7.5\% in terms of accuracy compared to 32-bit and mixed precision, respectively. Of the total of 33 experimental groups calculated with three optimizers, 31 decreased running time by more than 40\% compared to 32 bits while maintaining the accuracy to a similar level. Every computational time and test accuracy in 16-bit deep neural networks is better than those of 32-bit and mixed precision. 

\subsection{CNN Experiments}

16-bit CNN experiments were conducted to determine whether 16-bit is enough for training a more complex image classification problem and how numerically different it is from 32-bit. We used the CIFAR10 dataset and three convolutional neural networks (CNN) models: AlexNet, VGG16, and ResNet-34. All of these experiments were carried out using Adam in a 16-bit environment. 
Since the batch normalization (BN) layer is not implemented as an off-the-shelf module in 16 bits, the experiment was conducted focusing on CNN models that could be used without the batch normalization layer.  See the Appendix for treatment on 16-bit BN implementations.
\subsubsection{Training Time and Accuarcy}
Figure ~\ref{fig:opt} shows that 16-bit CNNs are better than 32-bit and mixed precision in terms of computational time and test accuracy. Figure ~\ref{fig:cnn} shows that 16-bit operations can also be applied to CNN models for image classification. We found that 16-bit CNNs maintain similar accuracy to 32-bit.

\colorbox{gray!20}{AlexNet} At the smallest batch size of 64, 16-bit AlexNet's top-1 and top-2 accuracy results are increased by 0.9\% and 0.6\% compared to 32-bit, respectively, and training time decreased by 29.1\%. The running time of 16-bit AlexNet is reduced by more than 29\% compared to its 32-bit counterpart. 

\colorbox{gray!20}{VGG-16} In terms of top-1 and top-2 accuracy, 16-bit VGG16 increased by 1.6\% and 0.6\% compared to 32-bit, respectively, and learning time decreased by 43.7\%. All computational speeds of 16-bit VGG16 were decreased by more than 40\% compared to 32-bit.

\colorbox{gray!20}{ResNet-34} 16-bit ResNet-34 decreased 0.6\% and 0.4\% in terms of top-1 and top-2 accuracy, and computational time decreased 44.7\% compared to 32-bit. The running time of 16-bit ResNet-34 were reduced by more than 39\%.\\
Overall, a 2.6\% (AlexNet, Batch Size: 384) decrease in 16-bit top1-accuracy compared to 32-bit in 33 experimental sets was the largest decrease, and the largest increase is 2.7\% (ResNet-34, Batch Size: 192). The smallest decrease in running time is by 29.1\% (AlexNet, Batch Size: 64) compared to 32 bits, and 45.6\% (VGG16, Batch Size: 288) the largest decrease. In other words, all 16-bit experimental groups used in all experiments decreased their computational speed by at least 29.1\%.

This experiment shows that by using 16-bit operations in image classification with CNN models, the running time (whether training or testing) can be greatly reduced while maintaining similar or higher accuracy compared to 32-bit operations. Our results demonstrate the efficiency of neural networks in training CNN models in a low-precision setting. 

\subsubsection{Model Size}
The preservation of the trained model weights is an important aspect of contemporary deep learning. 
\begin{table}[hbt!]
	\centering
	\footnotesize
	\setlength{\tabcolsep}{3pt}
	\renewcommand{\arraystretch}{1.5}
	\caption{Sizes of saved CNN models in 16- and 32-bit.}
	\vskip 0.15in
	\label{tab:capacity}
	\begin{tabular}{ccc}
		\toprule
		Model & FP16 & FP32 \\
		\midrule
		AlexNet & 41.99 MB & 83.94 MB \\
		VGG16 & 67.34 MB & 134.61 MB \\
		ResNet-34 & 171.99 MB & 343.87 MB\\
		\bottomrule
	\end{tabular}
	%\vskip -0.1in
\end{table}

It is useful for model training and storage if a model with similar accuracy has less storage size. Table~\ref{tab:capacity} shows the stored model has about half the size of the 32-bit size, faithfully reflecting the half-precision size reduction. 16-bit neural network allows for reducing the size of the model by half while maintaining similar accuracy. This opens up possibilities for 16-bit models to afford more complex and accurate architectures to attain better results.

\subsection{Limitation and Discussion}

This subsection reports the  limitations we find while using the 16-bit neural network.

\paragraph{Light hyperparameter-tuning.} 

During our experiments, we do not need to tune hyperparameters of  16-bit neural network training, e.g. learning rates,  except the epsilon in the settings of the optimizers. 
The SGD~\cite{sgd} optimizer does not have epsilon, so it can be used directly. For  the other optimizers we have tested in our experiments,  RMSProp~\cite{rmsprop} and Adam~\cite{adam}, we will need to change epsilon, whose default value (in Tensorflow), 1E-7, can easily trigger significant inaccuracy for 16-bit training. 

As a detail, the parameter epsilon corresponds to $\epsilon$ in the weight updates below:
\begin{align*}
	\text{RMSProp}: w_t =& w_{t-1} - \eta \frac{g_t}{\sqrt{v_t} + \epsilon} \\
	\text{ADAM}: w_t =& w_{t-1} - \eta \frac{\hat{m_t}}{\sqrt{\hat{v_t}} + \epsilon}
\end{align*} 
These optimizers introduced $\epsilon$  in to enhance numerical stability, but  $\epsilon$ = 1E-7 in the denominator causes floating-point overflow when $v_t$ in RMSProp or $\hat{v_t}$ in Adam are close to 0. As mentioned previously, we set $\epsilon$ = 1E-3 for 16-bit training. 

\paragraph{Missing 16-bit Batch normalization.} Tensorflow's current batch normalization layer does not directly support pure 16-bit operations.  The batch normalization layer presumably originates from  mixed precision works,  which would cast 16-bit input values to 32-bit for calculation and then down-cast to 16 bits. This type conversion only results in runtime overhead due to the intermediate 32-bit computation.

Since this work aims to report results on pure 16-bit neural networks, we have to implement a 16-bit batch normalization layer on our own. Our implementation of the 16bit batch normalization can be found in Appendix.

\paragraph{Batch size.} Our CNN experiments show that the accuracy of  16-bit neural networks decreases in larger batch sizes. This can be due to the precision loss incurred when averaging the cost over a larger number of samples in the mini-batches. This  is probably a minor limitation since neural networks in memory-constrained environments usually do not use large batch sizes.

Despite these limitations, we have confirmed that 16-bit neural networks learn as well as 32-bit with only  minor fine-tuning ($\epsilon$ in the optimizers). 
Thus, we believe ML practitioners can readily  benefit from 16NN when it comes to solving op optimization problems since it is faster, less memory-consuming, yet achieves similar accuracy as the 32-bit.
Instead of spending the same amount of time and space as 32-bit, we can form a network with an enhanced cost-to-benefit ratio. 

%% file: conclusion.tex
In this work, we have shown that for classification problems, the 16-bit floating-point neural network can have an accuracy close to the 32-bit. We have proposed a conjecture on their accuracy closeness and have validated it theoretically and empirically. Our experiments also suggest a small amount of accuracy gain, possibly due to the regularizing effect of
lowering the precision. These findings show that it is much safer and more efficient to use 16-bit precision
than what is commonly perceived: the runtime and memory consumption is lowered significantly with little or no
loss in accuracy.
In the future, we plan to expand our work to verify similar characteristics in other types of architectures
and problems such as generative models and regression.